\documentclass[conference]{IEEEtran}

\usepackage{graphicx} 
\usepackage{cite}
\usepackage{subfig}
\usepackage{url}      
\usepackage[cmex10]{amsmath}
\usepackage{amssymb,amsthm}

\usepackage{algorithm}
\usepackage{algorithmic}

\newtheorem{theorem}{Theorem}

\begin{document}
%
\title{Crowd-ML: A Privacy-Preserving Learning Framework for 
a Crowd of Smart Devices}

\author{\IEEEauthorblockA{Jihun Hamm, Adam Champion, Guoxing Chen, Mikhail Belkin, Dong Xuan}
\IEEEauthorblockA{Department of Computer Science and Engineering\\
The Ohio State University\\
Columbus, OH 43210\\
\\
}
}

%


\maketitle

\begin{abstract}
  Smart devices with built-in sensors, computational
  capabilities, and network connectivity have become increasingly
  pervasive. The crowds of smart devices offer opportunities to collectively
  sense and perform computing tasks in an unprecedented scale. 
  This paper presents Crowd-ML, a privacy-preserving machine learning
  framework for a crowd of smart devices, which can solve a wide range of
  learning problems for crowdsensing data with differential privacy guarantees.
  Crowd-ML endows a crowdsensing system with an ability to learn classifiers or
  predictors online from crowdsensing data privately
  with minimal computational overheads on devices and servers, 
  suitable for a practical and large-scale employment of the framework. 
  We analyze the performance and the scalability of Crowd-ML, and implement the
  system with off-the-shelf smartphones as a proof of concept. 
  We demonstrate the advantages of Crowd-ML with real and simulated experiments 
  under various conditions.
\end{abstract}



%
\IEEEpeerreviewmaketitle

\section{Introduction}

\subsection{Crowdsensing}
\label{sec:crowdsensing}

Smart devices are increasingly pervasive in daily life. 
These devices are characterized by their built-in sensors (e.g.,
accelerometers, cameras, and, microphones), programmable computation ability, and
Internet connectivity via wireless or cellular networks. These
include stationary devices such as smart thermostats and mobile devices
such as smartphones or in-vehicle systems. 
More and more devices are also being interconnected, often referred to
as the ``Internet of Things.'' Inter-connectivity offers opportunities
for crowds of smart devices to collectively sense and compute
in an unprecedented scale. 
Various applications of crowdsensing have been proposed, including
personal health/fitness monitoring, environmental sensing, and monitoring 
road/traffic conditions (see Section~\ref{sec:crowdsensing}), 
and the list is currently expanding.

Crowdsensing is used primarily for collecting and analyzing aggregate data 
from a population of participants. 
However, more complex and useful tasks can be performed 
beyond calculation of aggregate statistics, 
by using machine learning algorithms on crowdsensing data.
Examples of such tasks include:
learning optimal settings of room temperatures for smart thermostats; 
predicting user activity for context-aware services and physical monitoring;
suggesting the best driving routes;
recognizing audio events from microphone sensors.
Specific algorithms and data types for these tasks are different, 
but they can all be trained in standard unsupervised or supervised learning settings:
given sensory features (time, location, motion, environmental measures, etc.),
train an algorithm or model that can accurately predict a variable of interest 
(temperature setting, current user activity, amount of traffic, audio events, etc.). 
Conventionally, crowdsensing and machine learning are performed as two separate processes:
devices passively collect and send data to a central location, and 
analyses or learning procedures are performed at the remote location.
However, current generations of smart devices have computing capabilities
in addition to sensing. 
In this paper, we propose to utilize computing capabilities of smart devices,
and integrate sensing and learning processes together into a crowdsensing system.
As we will show, the integration allows us to design a system with better privacy 
and scalability.

\subsection{Privacy}

Privacy is an important issue for crowdsensing applications.
By assuring participants' privacy, a crowdsensing system can appeal to
a larger population of potential participants, which increases the utility of 
such a system.
However, many crowdsensing systems in the literature do not employ
any privacy-preserving mechanism (see Section~\ref{sec:private learning}),
and existing mechanisms used in crowdsensing (see ~\cite{participatorysensing2})
are often difficult to compare qualitatively across different systems or 
data types. 
In the last decade, differential privacy has gained popularity as a formal and
quantifiable measure of privacy risk in data publishing  \cite{Dwork:2004:CRYPTO,Dwork:2006:TC,Dwork:2006:ALP}.
Briefly, differential privacy measures how much the outcome of a procedure changes
probabilistically by presence/absence of any single subject in the original data.
The measure provides an upper bound on privacy loss {\it regardless of} the content
of data or any prior knowledge an adversary might have.
While differential privacy has been applied in data publishing and machine learning,
(see Section~\ref{sec:private learning}), it has not been broadly adopted in
crowdsensing systems. 
In this paper, we integrate differentially private mechanisms into the crowdsensing
system as well, which can provide strong protection against
various types of possible attacks (see Section~\ref{sec:privacy mechanism}).

\begin{figure}[th]
\centering
\includegraphics[width=0.9\linewidth]{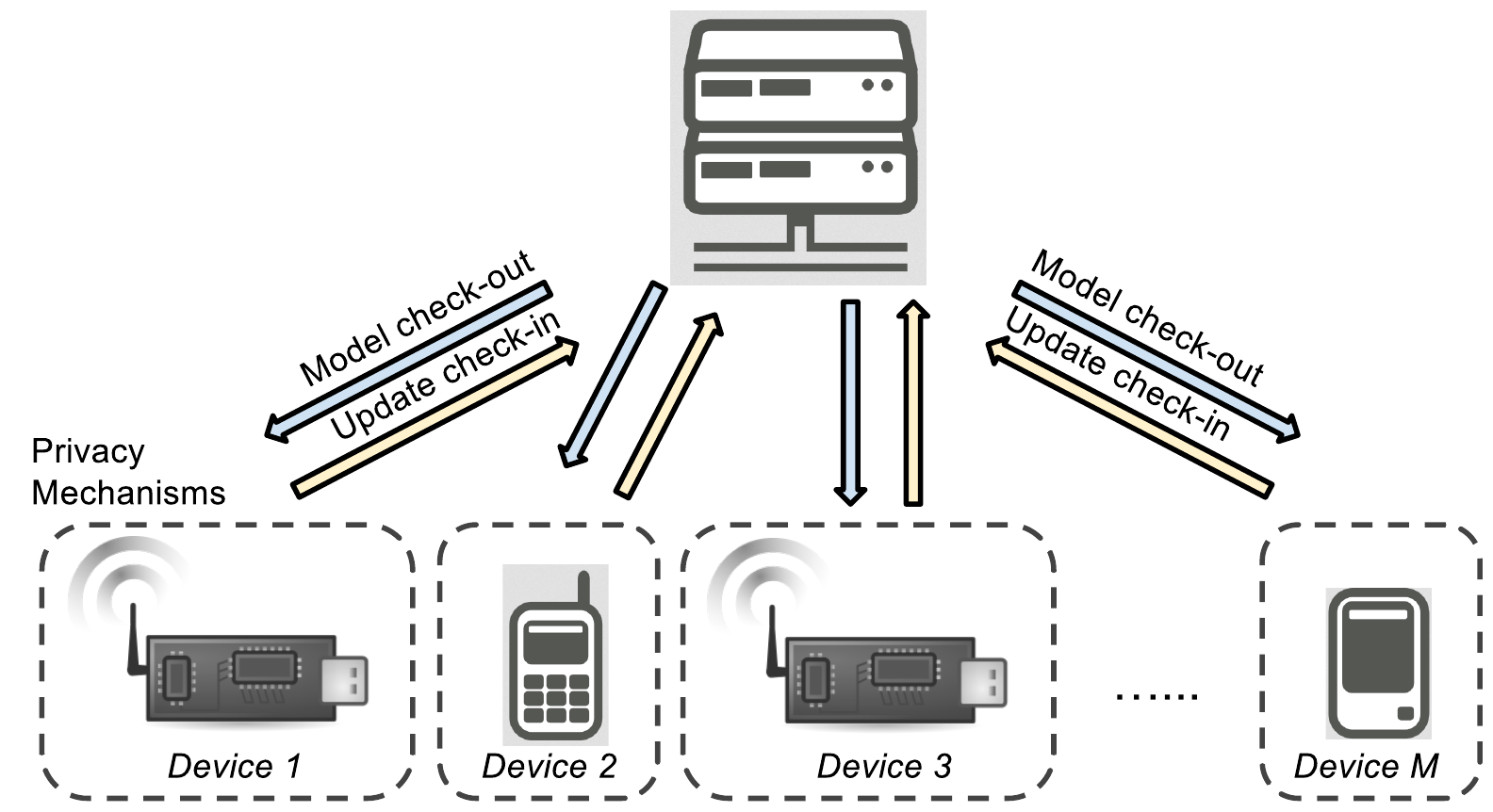}
\caption{Crowd-ML consists of a server and a number of smart devices.
The system integrates sensing, learning, and privacy mechanisms together, 
to learn a classifier or predictor from device-generated data in an online and
distributed way, with formal privacy guarantees.
}
\label{fig:system}
\end{figure}

\subsection{Proposed work}\label{sec:proposed work}

This paper presents Crowd-ML, a privacy-preserving machine learning
framework for crowdsensing system that consists of a server and smart devices
(see Fig.~\ref{fig:system}).  
Crowd-ML is a distributed learning framework that integrates sensing, learning, and
privacy mechanisms together, and can build classifiers or predictors of interest
from crowdsensing data using computing capability of devices with formal privacy
guarantees.

Algorithmically, Crowd-ML learns a classifier or predictor by a distributed incremental
optimization. 
Optimal parameters of a classifier or predictor are found by minimizing
the risk function associated with a given task \cite{Vapnik:2000} (see Section~\ref{sec:system} for details). 
Specifically, the framework finds optimal parameters by incrementally minimizing the risk
function using a variant of stochastic (sub)gradient descent (SGD)~\cite{Robbins:1951}. 
Unlike batch learning, SGD requires only the gradient information to be communicated
between devices and a server, which has two important consequences: 
1) computation load can be distributed among the devices,
enhancing scalability of the system; 
2) private data of the devices need not be communicated directly, enhancing privacy.
By exploiting these two properties, Crowd-ML efficiently learns a classifier or
predictor from a crowd of devices, with a guarantee of $\epsilon$-differential privacy.
Differential privacy mechanism is applied locally on each device, 
using Laplace noise for the gradients and exponential mechanisms for other information
(see Section~\ref{sec:privacy mechanism}).

We show advantages of Crowd-ML by analyzing its scalability and privacy-performance
trade-offs (Section~\ref{sec:analysis}), and by testing the framework with 
demonstrative tasks implemented on Android smartphones and in simulated environments
 under various conditions (see Section~\ref{sec:evaluation}).

In summary, we make the following contributions:
\begin{itemize}
\item We present Crowd-ML, a general framework for machine learning with
  smart devices from crowdsensing data, with many potential applications.
\item We show differential privacy guarantees of Crowd-ML that provide a 
strong privacy mechanism against various types of attacks in crowdsensing.
To the best of our knowledge, Crowd-ML is the first general framework that
integrates sensing, learning, and differentially private mechanisms for crowdsensing.
\item We analyze the framework to show that the cost of privacy preservation 
can be minimized and that the computational and communication overheads on devices
are only moderate, allowing a large-scale deployment of the framework.
\item We implement a prototype and evaluate the framework with a
demonstrative task in a real environment as well as large-scale experiments
in a simulated environment.
\end{itemize}


The remainder of this paper is organized as follows.
We first review related work in Section~\ref{sec:related work}.
Section~\ref{sec:crowd-ml} describes the Crowd-ML framework. 
Section~\ref{sec:analysis} analyzes Crowd-ML in terms of privacy-performance trade-off,
computation, and communication loads. 
Section~\ref{sec:evaluation} presents an implementation of Crowd-ML and 
experimental evaluations.
We discuss remaining issues and conclude in 
Section~\ref{sec:conclusion}. 

\section{Related Work}\label{sec:related work}

Crowd-ML integrates distributed learning algorithms and differential privacy
mechanisms into a crowdsensing system.
In this section, we review related work in crowdsensing and learning systems, and
privacy-preserving mechanisms.

\subsection{Crowdsensing and learning}
\label{sec:crowdsensing}

There is a vast amount of work in crowdsensing, and we focus on 
the system aspect of previous work with representative papers 
(we refer the reader to survey papers
\cite{Lane:2010:IEEECOMM} and \cite{participatorysensing2}).
Crowdsensing systems aim to achieve mass collection and mining of environmental
and human-centric data such as social interactions, political issues of interest,
exercise patterns, and people's impact on the environment \cite{humancentricsensing}.
Examples of such systems include
Micro-Blog \cite{microblog}, PoolView \cite{poolview}, BikeNet
\cite{bikenet}, and PEIR \cite{peir}.
Data collected by crowdsensing can also be used to mine high-level patterns
or to predict variables of interest using machine learning.
Applications of learning applied to crowdsensing include  
learning of bus waiting times \cite{howlongtowait} and 
recognizing user activities (see \cite{Lara:2013} for a review).
Jigsaw \cite{jigsaw} and Lifestreams \cite{lifestreams} also use pattern recognition 
in sensed data from mobile devices.
From the system perspective, these work use devices to passively sense and send data 
to a central server on which analyses take place, which we will refer to as the 
{\it centralized} approach.
In contrast, sensing and learning can be performed purely inside each device without
a server, which we call the {\it decentralized} approach.
For example, SoundSense \cite{Lu:2009} learns a classifier on a smartphone 
to recognized various audio events without communicating with the back-end.
Mixed centralized and decentralized approaches are also used in \cite{cque,ace},
where a portion of computation is performed off-line on a server.
CQue \cite{cque} provides a query interface for privacy-aware probabilistic
learning of users' contexts, and ACE \cite{ace} uses static association
rules to learn users' contexts.
System-wise, our work differs from those centralized or decentralized approaches
in that we use a {\it distributed} approach to perform learning by devices and server
together, which improves privacy and scalability of the system. 
We are not aware of any other crowdsensing system that takes a similar approach. 
Also, the cited papers are oriented towards novel applications, 
but our work focuses on a general framework for learning a wide range of algorithms
and applications. 

Crowd-ML also builds on recent advances in incremental distributed learning \cite{Agarwal:2011:NIPS,Dekel:2011:ICML},
which show that a near-optimal convergence rate is achievable despite communication
delays. A privacy-preserving stochastic gradient descent method is presented briefly in \cite{Song:2013:IEEE}.
Unlike the latter, we presents a complete framework for privacy-preserving
multi-device learning, with performance analysis and demonstrations in real environments.

\subsection{Privacy-preserving mechanisms}
\label{sec:private learning}

Privacy is an important issue in data collection and analysis.
In particular, preserving privacy of users' locations has been studied by many
researchers (see \cite{Krumm:2009} for a survey). 
To preserve privacy of general data types formally, several mechanisms such as 
$k$-anonymity \cite{kanonymity} and secure multiparty computation \cite{Yao:1982:FOCS}
have been proposed, for data publishing \cite{Fung:2010:CSUR} and also for 
participatory sensing \cite{participatorysensing2}.
Recently, differential privacy \cite{Dwork:2004:CRYPTO,Dwork:2006:TC,Dwork:2006:ALP}
has addressed several weaknesses of $k$-anonymity \cite{Ganta:2008:SIGKDD}, 
and gained popularity as a
quantifiable measure of privacy risk. Differential privacy has been used for 
privacy-preserving data analysis platform \cite{Mcsherry:2009:SIGMOD}, 
for sanitization of learned models parameters from data
\cite{Chaudhuri:2011:JMLR}, and for privacy-preserving data mining from distributed
time-series data\cite{Lu:2009}. 
So far, formal and general privacy mechanisms have not been adopted broadly 
in crowdsensing. Among the crowdsensing systems cited in the previous section (~\cite{microblog,poolview,bikenet,peir,howlongtowait,a3c, phoneactivitydataset,jigsaw,lifestreams,Lu:2009,cque,ace}),
only \cite{poolview,peir,cque} provide privacy mechanisms, of which only
\cite{poolview} address the privacy more formally.
To our best knowledge, Crowd-ML is the first framework to provide formal
privacy guarantees in general crowd-based learning with smart devices.

\section{Crowd-ML}\label{sec:crowd-ml}

In this section, we describe our Crowd-ML in detail: system, algorithms, and
privacy mechanisms.

\subsection{System and workflow}
\label{sec:system}

The Crowd-ML system consists of a server and multiple smart devices
that are capable of sensory data collection, numerical computation,
and communication over a public network with the server (see Fig.~\ref{fig:system}).
The goal of Crowd-ML is to learn a classifier or predictor of interest  
from crowdsensing data collectively by multiple devices.
A wide-range of classifiers or predictors can be learned by minimizing an empirical
risk associated with a given task, a common method in statistical learning \cite{Vapnik:2000}.
Formally, let $x \in \mathbb{R}^D$ be a feature vector from preprocessing sensory input
such as audio, video, accelerometer, etc, and $y$ be a target variable we aim to
predict from $x$, such as user activity. For regression, $y$ can be a real number
and for classification, $y$ is a discrete label $y \in \{1,...,C\}$ with $C$
classes. 
We define data as $N$ pairs of (feature vector, target variable)
generated i.i.d. from an unknown distribution by all participating devices up 
to present:
\begin{equation}\label{eq:private data}
\mathcal{D} = \{(x_1,y_1),..., (x_N,y_N)\}.
\end{equation}
Suppose we use a classifier/predictor $h(x;w)$ with a tunable parameter vector $w$,
and a loss function $l(y,h(x;w))$ to measure the performance of the classifier/predictor
with respect to the true target $y$. A wide range of learning algorithms can be
represented by $h$ and $l$, e.g., regression, logistic regression, and 
Support Vector Machine (see \cite{Bottou:2012} for more examples).
If there are $M$ smart devices, we find the optimal parameters $w$ of the classifier/predictor by minimizing the empirical risk over all $M$ devices:
\begin{equation}\label{eq:empirical risk}
\mathcal{R}(w) = \sum_{m=1}^M \frac{1}{|\mathcal{D}_m|} 
\sum_{(x,y) \in \mathcal{D}_m} l(h(x;w),y) + \frac{\lambda}{2}\|w\|^2,
\end{equation}
where $\mathcal{D}_m$ is a set of samples generated from device $m$ only,
and $\frac{\lambda}{2}\|w\|^2$ is a regularization term.
This risk function (\ref{eq:empirical risk}) can be minimized by many optimization
methods. 
In this work we use stochastic (sub)gradient descent (SGD) \cite{Robbins:1951:AMS} 
which is one of the simplest optimization methods and is also suitable for
large-scale learning \cite{Bottou:1998,Bottou:2012}.
SGD minimizes the risk by updating $w$ sequentially 
\begin{equation} \label{eq:sgd}
w{(t+1)} \leftarrow \Pi_{\mathcal{W}} \left[w{(t)} - \eta{(t)} g(t)\right],
\end{equation}
where $\eta{(t)}$ is the learning rate, and
$g(t)$ is the gradient of the loss function
\begin{equation}\label{eq:subgradient}
g = \nabla_w l(h(x;w),y),
\end{equation}
evaluated with the sample $(x,y)$ and the current parameter $w(t)$.
We assume the parameter domain $\mathcal{W}$ is a $d$-dimensional ball of 
some large radius $R$, and the projection is $\Pi_{\mathcal{W}} = \min (1, R/\|w\|) w$. 
By default, we use the learning rate
\begin{equation}\label{eq:learning rate}
\eta^{(t)} = \frac{c}{\sqrt{t}},
\end{equation}
where $c$ is a constant hyperparameter. 
When computing gradients, we use a `minibatch' of $b$ samples to 
compute the averaged gradient
\begin{equation}\label{eq:minibatch}
\tilde{g} = \frac{1}{b} \sum_i \nabla_w l(h(x_i;w),y_i),
\end{equation}
which plays an important role in the performance-privacy trade-off and the scalability
(Section~\ref{sec:analysis}).
In Crowd-ML, risk minimization by SGD is performed by distributing the
main workload (=computation of averaged gradients) to $M$ devices. Note that
each device generates data and compute gradients using its own data. 
The workflow is described in Fig.~\ref{fig:workflow}. 

\begin{figure}[thb]
\centering
\includegraphics[width=1.0\linewidth]{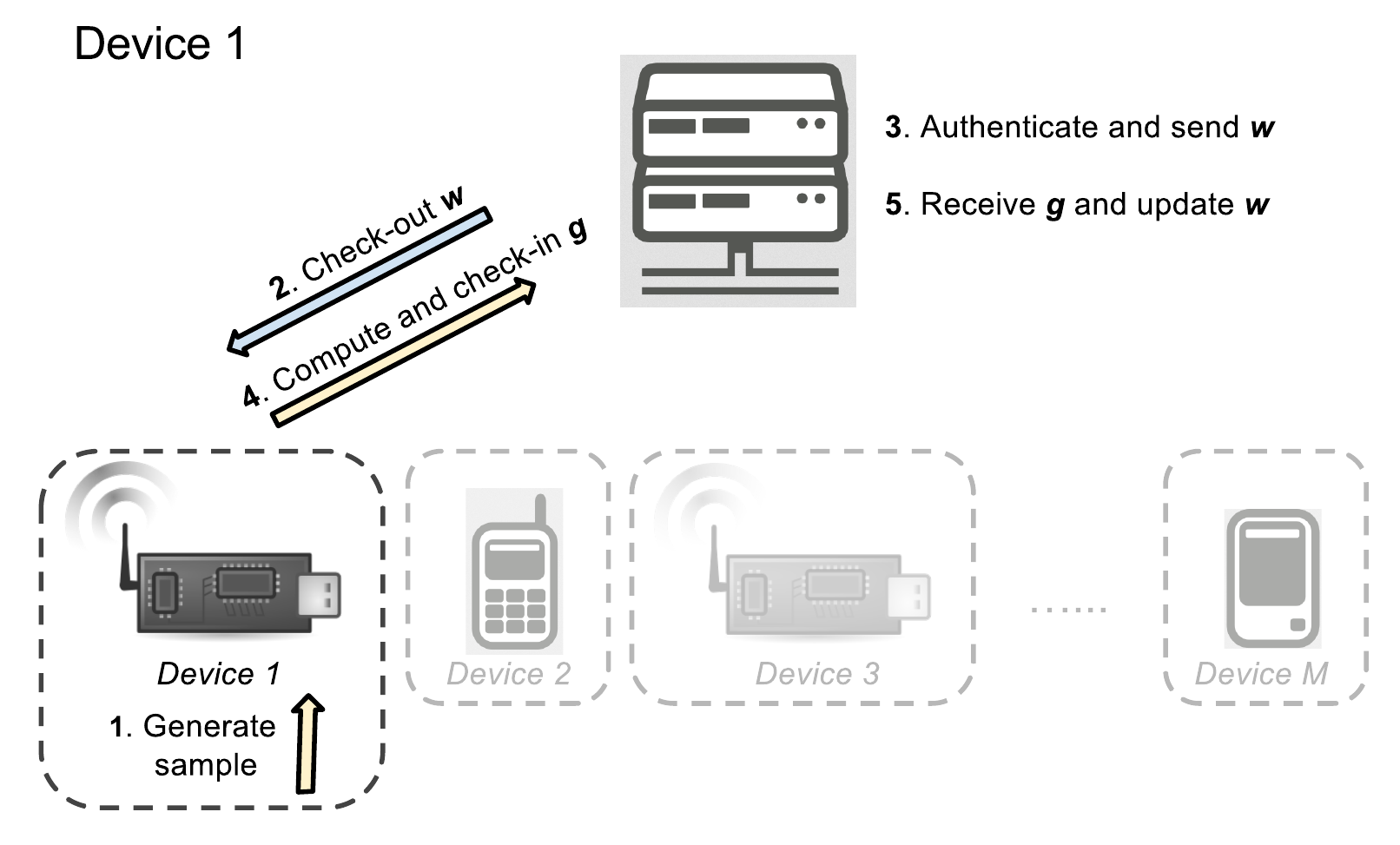}
\caption{Crowd-ML workflow.
1. A device preprocesses sensory data and generates a sample(s).
2. When the number of samples $\{(x,y)\}$ exceeds a certain number, 
the device requests current model parameters $w$ from the server.
3. The server authenticate the device and sends $w$.
4. Using $w$ and $\{(x,y)\}$, the device computes the gradient $g$ and 
send it to the server using privacy mechanisms.
5. The server receives the gradient $g$ and updates $w$. 
While one device is performing routines 1-5, another device(s) are allowed 
to perform the same routines asynchronously. 
Devices can join or leave the task at any time.
}
\label{fig:workflow}
\end{figure}

\subsection{Algorithms}
\label{sec:algorithm}

\begin{algorithm}[tb] \caption{Device side} \label{alg:device}
{\it Input}: privacy levels $\epsilon_g,\epsilon_e,\epsilon_{y^k}$, minibatch size $b$, 
max buffer size $B$, classifier model ($C$, $h$, $l$, $\lambda$ from Eq.~(\ref{eq:empirical risk}))\\
{\it Init}: set $n_s=0$, $n_e = 0$, $n_y^k=0,\;k=1,...,C$\\
{\it Communication to server}: $\hat{g}, n_s, \hat{n}_e, \hat{n}_y^k$	\\
{\it Communication from server}: $w$ \\
{\bf Device Routine 1}
\begin{algorithmic}
\IF{$n_s \geq B$}
	\STATE stop collection to prevent resource outage
\ELSE
	\STATE receive a sample $(x,y)$ (in a regular interval or triggered by events),
	and add to the secure local buffer
	\STATE $n_s = n_s + 1$
\ENDIF
\IF{$n_s \geq b$}
	\STATE checkout $w$ from the server via https
	\STATE call Device Routine 2.
\ENDIF
\end{algorithmic}

{\bf Device Routine 2}
\begin{algorithmic}
\STATE Using $w$ from the server and $\{(x,y)\}$ from the local buffer,
\FOR{$i = 1,...,n_s$}
	\STATE make a prediction $y^\mathrm{pred}=h(x_i;w)$
	\STATE $n_y^{(y^{}_i)} = n_y^{(y^{}_i)} + 1$
	\STATE $n_e = n_e + I[y^\mathrm{pred}_i \neq y^{}_i]$
	\STATE Incur a loss $l(y^\mathrm{pred},y_i)$
	\STATE Compute a subgradient $g_i = \nabla_w l(h(x_i;w))$
\ENDFOR
\STATE Compute the average $\tilde{g} = \frac{1}{n_s} \sum_i g_i + \lambda w$
\STATE Sanitize data with Device Routine 3
\STATE Checkin $\hat{g},\;n_s,\;\hat{n}_e\;\hat{n}_y^k,\;k=1,...,C$ with server via https
\STATE Reset $n_s = 0,\; n_e = 0,\; n_y^k = 0,\;k=1,...,C$
\end{algorithmic}

{\bf Device Routine 3}
\begin{algorithmic}
\STATE Sample $\hat{g} = \tilde{g} + z$ from Eq.~(\ref{eq:perturbation of g})
\STATE Sample $\hat{n}_e = n_e + z$ from Eq.~(\ref{eq:perturbation of n_e})
\STATE Sample $\hat{n}_y^k = n_y^k + z,\;k=1,...,C$ from Eq.~(\ref{eq:perturbation of n_y^k})
\end{algorithmic}
\end{algorithm}

\begin{algorithm}[tb]
   \caption{Sever side} \label{alg:server}
{\it Input}: number of devices $M$, learning rate schedule $\eta(t),\;t=1,2,...,T_{\max}$, desired error $\rho$, 
classifier model ($C$, $h$, $l$, $\lambda$ from Eq.~(\ref{eq:empirical risk}))\\
{\it Init}: $t=0$, randomized $w$, $N_s^m=0,\;N_e^m=0,\;N_y^{k,m},\;m=1,...,M,\;k=1,...,C$\\
{\it Stopping criteria}: $t \geq T_{\max}$ or $\frac{\sum_m^M N_e^m}{\sum_m^M N_s^m} \leq \rho$ 
\\
{\bf Server Routine 1}
\begin{algorithmic}
\WHILE{Stopping criteria not met}
	\STATE Listen to and accept checkout requests
	\STATE Authenticate device
	\STATE Send current parameters $w$ to device
\ENDWHILE
\end{algorithmic}
{\bf Server Routine 2}
\begin{algorithmic}
\WHILE{Stopping criteria not met}
	\STATE Listen to and accept checkin requests
	\STATE Authenticate device (suppose it is device $m$)
	\STATE Receive $\hat{g}$, ${n}_s$, $\hat{n}_e$, $\hat{n}_y^k,\;k=1,...,C$.
	\STATE $N_s^m = N_s^m + n_s$
	\STATE $N_e^m = N_e^m + \hat{n}_e$
	\STATE $N_y^{k,m} = N_y^{k,m} + \hat{n}_y^{k}$
	\STATE $w = w - \eta(t) \hat{g}$
	\STATE $t = t + 1$
\ENDWHILE
\end{algorithmic}
\end{algorithm}

Crowd-ML algorithms are presented in Algorithms~\ref{alg:device} and~\ref{alg:server}.
Device Routine 1 collects samples. When the number of samples
reaches the minibatch size $b$, the routine tries to checks out the current model
parameters $w$ from the server and calls Device Routine 2. 
Device Routine 2 computes the averaged gradient from the stored samples and $w$
received from the server, sanitizes information by Device Routine 3, 
and sends the sanitized information to the server.
Device Routine 3 uses Laplace noise and exponential mechanisms (in the next section)
to sanitize the averaged gradient $\hat{g}$, 
the number of misclassified samples $\hat{n}_e$ and the label counts $\hat{n}_y^k$. 
Device Routines 1-3 are performed independently and asynchronously by multiple devices.

Server Routine 1 sends out current parameters $w$ when requested
and Server Routine 2 receives checkins ($\hat{g}$, $n_s$, $\hat{n}_e$,$\hat{n}_y^k$)
from devices when requested. The whole procedure ends when the total number of
iteration exceeds a maximum value $T_{\max}$, or the overall error is below a threshold $\rho$.

\emph{Remark 1:} In Device Routine 1, if check-out fails, the device keeps collecting
samples and retries check-out later. A prolonged period of network outage for a device
can make the parameter outdated for the device, but it does not affect the
overall learning critically. 
Similarly, failure to check-in information with server in Device Routine 2 is non-critical.

\emph{Remark 2:} In Device Routine 2, we can randomly set aside a small portion of samples
as test data. 
In this case, the misclassification error is computed only
from these held-out samples, and their gradients will not be used in the average $\hat{g}$. 

\emph{Remark 3:} In Server Routine 2, more recent update methods  \cite{Nesterov:2009:MP,Roux:2012:NIPS}
can be used in place of the simple update rule (\ref{eq:sgd})
without affecting differential privacy nor changing device routines. 
Similarly, adaptive learning rates \cite{Duchi:2010:COLT,Schaul:2013:ICML} can be
used in place of (\ref{eq:learning rate}), which can provide a robustness 
to large gradients from outlying or malignant devices. 

\subsection{Privacy mechanism}
\label{sec:privacy mechanism}

In crowdsensing systems, private data of users can be leaked by many ways.
System administrators/analysts can violate the privacy intentionally, or they may
leak private information unintentionally when publishing data analytics.
There are also more hostile types of attacks: 
by malignant devices posing as legitimate devices,
by hackers poaching data stored on the server or eavesdropping on communication
between devices and servers.
Instead of preserving privacy separately for each attack type, 
we can preserve privacy from all these attacks by a {\it local} privacy-preserving
mechanism that is implemented on each device and sanitizes any information before
it leaves the device. 
A local mechanism assumes that an adversary can potentially access all communication
between devices and the server, which subsumes the other attack attacks. 
This is because the other forms of data that are 1) visible to malignant device, 
2) stored in the server, or 3) released in public, are all derived from what is
communicated between devices and the server. 
We adopt a local $\epsilon$-differential privacy as a quantifiable
measure of privacy in Crowd-ML.
Formally, a (randomized) algorithm which takes data $\mathcal{D}$ as input 
and outputs $f$, is called $\epsilon$-differentially private if
\begin{equation}
\frac{P(f(\mathcal{D}) \in \mathcal{S})}
{P(f(\mathcal{D}') \in \mathcal{S})} \leq e^{\epsilon}
\end{equation}
for all measurable $\mathcal{S} \subset \mathcal{T}$ of the output range,
and for all data sets $\mathcal{D}$ and $\mathcal{D}'$ differing in a single item.
That is, even if an adversary has the whole data $\mathcal{D}$ except a single item,
it cannot infer much more about that item from the output of the algorithm $f$.
A smaller $\epsilon$ makes such an inference more difficult, and therefore makes
the algorithm more private-preserving.
When the algorithm outputs a real-valued vector $f \in \mathbb{R}^D$, its global
sensitivity can be defined by
\begin{equation}
S(f) = \max_{\mathcal{D},\mathcal{D}'} \|f(\mathcal{D}) - f(\mathcal{D}')\|_1.
\end{equation}
where $\|\cdot\|_1$ is the $L_1$ norm.
A basic result from the definition of differential privacy is that
a vector-valued function $f$ with sensitivity $S(f)$ can be made 
$\epsilon$-differentially private {\cite{Dwork:2006:TC}}
by adding an independent Laplace noise vector $z$\footnote{As a variant, 
$(\epsilon,\delta)$-differential privacy can be achieved by adding Gaussian noise.}
\begin{equation}
P(z) \propto e^{-\frac{\epsilon}{S(f)}\|z\|_1}.
\end{equation} 
In Crowd-ML, we consider $\epsilon$-differential privacy of any single 
(feature,label)-sample, revealed by communications from all devices to the server,
which are the gradients $\tilde{g}$, the numbers of samples $n_s$, the number of
misclassified samples $n_e$, and the labels counts $n_y^k$
\footnote{The communication from the server to devices $\{w(t)\}$ can 
be reconstructed by (\ref{eq:sgd}) from $\{g(t)\}$, and therefore is redundant to
consider.}. 
\begin{table}[htb]
\caption{Multiclass logistic regression}
\label{tbl:learning algorithms}
\begin{center}
\begin{tabular}{|p{0.10\linewidth}|p{0.80\linewidth}|}
\hline
Prediction & $\arg\max_k w_k'x$ \\
Risk & $\mathcal{R}(w)=\frac{1}{N} \sum_{i}[ -w_{y_i}'x_i + \log\sum_l e^{w_l'x_i}] + \;\;\;\;\frac{\lambda}{2}\sum_k \|w_k\|^2 $ \\
Gradient & $\nabla_{w_k} \mathcal{R} = \frac{1}{N} \sum_{i} x_i [ -I[y_i=k] + P(y=k|x_i)] + \;\;\;\;\lambda w_k$
\\ 
\hline
\end{tabular}
\end{center}
\end{table}
The amount of noise required depends on the choice of loss functions. 
We compute this value for multiclass logistic regression (Table~\ref{tbl:learning algorithms}),
but it can be computed similarly for other loss functions as well. 
By adding element-wise independent Laplace noise $z$ to averaged gradients $\tilde{g}$ 
\begin{equation}\label{eq:perturbation of g}
\hat{g} = \frac{1}{b} \sum_i g_i + z,\;\;P(z) \propto e^{-\frac{\epsilon_g b}{4}|z|},
\end{equation}
we have the following privacy guarantee:
\begin{theorem}[Averaged gradient perturbation]
The transmission of $\tilde{g}$ by Eq.~(\ref{eq:perturbation of g}) is 
$\epsilon_g$-differentially private.
\end{theorem}
\noindent See Appendix~\ref{sec:proof1} for proof.

To sanitize $n_e$ and $n_y^k$, we add `discrete' Laplace noise \cite{Inusah:2006:JSPI}
 as follows:  
\begin{eqnarray}\label{eq:perturbation of n_e}
\hat{n}_e &=& n_e + z,\;P(z) \propto e^{-\frac{\epsilon_e}{2}|z|},
\label{eq:perturbation of n_e}\\
\hat{n}_y^k &=& n_y^k + z,\;P(z) \propto e^{-\frac{\epsilon_{y^k}}{2} |z|},
\label{eq:perturbation of n_y^k}
\end{eqnarray}
where $z = 0,\pm 1,\pm 2,...$.
These mechanisms has the following privacy guarantees:
\begin{theorem}[Error and label counts]
The transmission of $n_e$ and $n_y^k$ by Eqs.~(\ref{eq:perturbation of n_e})
 and (\ref{eq:perturbation of n_y^k}) is $\epsilon_e$- and $\epsilon_{y^k}$-
differentially private, respectively.
\end{theorem}
\noindent See Appendix~\ref{sec:proof2} for proof.

Practically, a system administrator chooses $\epsilon$ depending on the desired level
of privacy for the data collected.
A small $\epsilon (\to 0)$ may be used for data that users deem highly private
such as current location, and a large $\epsilon (\to \infty)$ may be used for 
less private data such as ambient temperature. 


\section{Analysis}\label{sec:analysis}

In this section, we analyze the privacy-performance trade-off and the scalability of
Crowd-ML. 
As discussed in Related Work, most existing crowdsensing systems use
purely centralized or purely decentralized approaches, while Crowd-ML uses a 
distributed approach.
By design, Crowd-ML achieves differential privacy with little loss of performance
($O(1/b)$), only moderate computation load due to its simple optimization method,
and reduced communication load and delay ($O(1/b)$), where $b$ is the minibatch size.

\subsection{Privacy vs Performance}

Privacy costs performance:
the more private we make the system, the less accurate
the outcome of analysis/learning is.
From Theorem~1, Crowd-ML is $\epsilon$-differentially private by perturbing averaged
gradients. 
The centralized approach can also be made $\epsilon$-differentially private by
feature and label perturbation (Appendix~\ref{sec:proof3}). 
Below we compare the impact of privacy on performance between the centralized and 
Crowd-ML.
The performance of an SGD-based learning can be represented by its rate of convergence
to the optimal value/parameters $\mathbb{E}[l(w(t)) - l(w^{\ast})]$ at iteration $t$,
which in turn depends on the properties of the loss $l(\cdot)$ 
(such as Lipschitz-continuity and strong-convexity) and the step size
$\eta(t)$, with the best known rate being $O(1/t)$ \cite{Nemirovski:2009}.
When other conditions are the same, the convergence rate is roughly proportional 
$\mathbb{E}[l(w(t)) - l(w^{\ast})] \propto G^2$ to the amount of noise 
in the estimated gradient $G^2 = \sup_t \mathbb{E}[ \|\hat{g}(t)\|^2]$
\cite{Shamir:2013}.
For Crowd-ML, we have from (\ref{eq:perturbation of g}) 
\begin{equation}
\mathbb{E}[\|\hat{g}\|^2] = \mathbb{E}[\|\tilde{g}\|^2] + \mathbb{E}[\|z\|^2]
= \frac{1}{b}\mathbb{E}[\|g\|^2] + \frac{32D}{(b \epsilon_g)^2},
\end{equation}
where the first term is the amount of noise due to sampling,
and the latter is due to Laplace noise mechanism with $D$-dimensional features.
By choosing a large enough batch size $b$, the impact of sampling noise and
Laplace noise can be made arbitrarily small\footnote{although a larger batch size 
means fewer updates given the same number of samples $N$, and too large a batch size
can negatively affect the convergence rate (see \cite{Cotter:2011} for discussion).}.
In contrast, the centralized approach has to add Laplace noise of {\it constant}
variance $\frac{8}{\epsilon^2}$ to each feature and perturb labels with a
constant probability (Appendix~\ref{sec:proof3}). 
Regardless of which optimization method is used (SGD or not), the centralized approach
has no means of mitigating the negative impact of constant noise on the accuracy
of learned model, which will be especially problematic with a small $\epsilon$. 

In the decentralized approach, a device need not interact with a server, and
is almost free of privacy concerns. However, the increased privacy comes at
the cost of performance. In Crowd-ML and the centralized approach, samples
pooled from all devices are used in the learning process, 
whereas in the decentralized approach, each device can use only a fraction ($\sim1/M$)
of samples. 
This undermines the accuracy of a model learned by the decentralized approach.
For example, it is known from the VC-theory for binary classification problems that
the upper-bound of the estimation error with a $1/M$-times smaller sample size is
$\sqrt{M/\log M}$-times larger \cite{Anthony:1999:Book}.

\subsection{Scalability}

Scalability is determined by computation and communication loads and latencies
on both device and server sides. 
We compare these factors between centralized, crowd, and decentralized learning approaches.


\subsubsection{Computation load}

For all three approaches, we assume the same preprocessing is performed
on each device to compute features from raw sensory input or metadata. 
On the device side, the centralized learning approach requires generation of 
Laplace noise per sample on the device. 
The crowd and the decentralized approaches perform partial and full learning on the device, respectively, and requires more processing.
Specifically, Crowd-ML requires computation of a gradient per sample, a vector
summation (for averaging) per sample, and generation of Laplace random noise
per minibatch.
A low-end smart device capable of floating-point computation can perform these operations. 
The decentralized learning approach can use any learning algorithms, including
SGD similar to Crowd-ML. However, if the decentralized approach is to make up for
the smaller sample size ($1/M$) compared to Crowd-ML, it may require
more complex optimization methods which results in higher computation load.
For all three approaches, the number of devices $M$ do not affect per-device computation
load. 
Computational load on the server is also different for these approaches. 
The centralized approach puts the highest load on the server, as all computations take
place on the server. In contrast, Crowd-ML puts minimal load on the server which is 
the SGD update~(\ref{eq:sgd}), since the main computation is performed distributed 
by the devices. 

\subsubsection{Communication load}

To process incoming streams of data from the device in time, the network and the
server should have enough throughput. 
The centralized learning approach requires $N$ number of samples to be 
sent over the network to the server.
For Crowd-ML with a minibatch size of $b$, devices send $N/b$ gradients altogether,
and receives the same number of current parameters, both of the same dimension 
as a feature vector. Therefore, the data transmission is reduced by a factor of
$b/2$ compared to the centralized approach.

\subsubsection{Communication latency}\label{sec:delay}

When using a public (and mobile) network, latency is non-negligible. 
In the centralized approach, latency may not be an issue, since the server 
need not required to send any real-time feedback to the devices. 
In Crowd-ML, latency is an issue that can affect its performance.
There are three possible delays that add up to the overall latency of communication:
\begin{itemize}\setlength{\itemsep}{0pt}
\item Request delay($\tau_{\mathrm{req}}$): time since the check-out request
from a device until the receipt of the request at the server
\item Check-out delay ($\tau_{\mathrm{co}}$):   
time since the receipt of a request at the server and the receipt of the parameter at
the device
\item Check-in delay ($\tau_{\mathrm{ci}}$): time since the
receipt of the parameters at the device until the receipt of the check-ins at the server
\end{itemize}
Due to delays, if a device checks out the parameter $w$ at time $t_0$ and checks in
the gradient $\hat{g}$ and the server receives $\hat{g}$ at time $t_0+\tau_{\mathrm{co}}+\tau_{\mathrm{ci}}$, the server may have already updated
the parameters $w$ multiple times using the gradients from other devices received
during this time period.
This number of updates is roughly 
$(\tau_{\mathrm{co}}+\tau_{\mathrm{ci}}) \times M F_s/b$,
where $M$ is the number of devices, $F_s$ is the data sampling rate per device,
and $1/b$ is the reduction factor due to minibatch. 
Again, choosing a large batch size $b$ relative to $M F_s$ can reduce the latency.
While exact analysis of impact of latency is difficult, there are several 
related results known in the literature without considering privacy.
Nedi{\'{c}} et al.~proved that delayed asynchronous incremental update converges with
probability 1 to an optimal value, assuming a finite maximum latency.
Recent work in distributed incremental update \cite{Agarwal:2011:NIPS,Dekel:2011:ICML}
also shows that a near-optimal convergence rate is achievable despite delays. 
In particular, Dekel et al.~\cite{Dekel:2011:ICML} shows that delayed incremental
updates are scalable with $M$ by adapting the minibatch size. 

\section{Evaluation}\label{sec:evaluation}

In this section, we describe a prototype of Crowd-ML implemented on off-the-shelf 
android phones and activity recognition experiments on android smartphones.
We also perform digit and object recognition experiments under varying conditions 
in simulated environments and demonstrate the advantages of Crowd-ML analyzed in
Section~\ref{sec:analysis}. 

\subsection{Implementation}\label{sec:implementation}

We implement a Crowd-ML prototype with three components: a Web portal,
commercial off-the-shelf smart devices, and a central server. 
On the device side, we implement Algorithm~\ref{alg:device} on commercial
off-the-shelf smartphones as an app using Android OS 4.3$+$.
Our prototype uses smartphones, but will be easily ported to other 
smart device platforms. 
On the server side, we implement Algorithm~\ref{alg:server} on a
Lenovo ThinkCentre M82 machine with a quad-core 3.2~GHz Intel Core
i5-3470 CPU and 4~GB RAM running Ubuntu Linux 14.04. The server runs
the Apache Web server (version 2.4) and a MySQL database (version 5.5). 

Also on the server side, our Crowd-ML prototype provides a Web portal over
HTTPS where users can browse ongoing crowd-learning tasks and join
them by downloading the app to their smart devices. To enhance
transparency, details of tasks (objective, sensory data collected,
labels collected, and learning algorithms used) and our privacy mechanisms
is explained.  It also displays timely statistics about crowd-learning
applications such as error rates and activity label distributions,
which are differentially private. 
We implement the portal in Python using the Django\footnote{\url{http://www.djangoproject.com}}
Web application framework and Matplotlib\footnote{\url{http://matplotlib.org}}
for statistical visualization.

\subsection{Activity Recognition in Real Environments}
In this experiment, we perform activity recognition on smart devices. 
The purpose of this demonstration is to show Crowd-ML working in
a real environment, so we choose a simple task of recognizing three types of 
user activities (``Still'', ``On Foot'', and ``In Vehicle''). 
We install a prototype Crowd-ML application on 7
smartphones (Galaxy Nexus, Nexus S, and Galaxy S3) running Android 4.3
or 4.4. 
The seven smartphones are carried by college students and faculty
over a period of a few days.  The devices' triaxial accelerometers
are sampled at 20~Hz.
In this demonstration, we avoid manual annotation of activity labels 
to facilitate data acquisition, 
and instead use Google's activity recognition service 
to obtain ground truth labels.
Acceleration magnitudes $|a| = \sqrt{a_x^2 + a_y^2 + a_z^2}$ are
computed continuously over 3.2~s sliding windows. Feature extraction
is performed by computing the 64-bin FFT of the acceleration magnitudes.
We set the sampling rate $F_s = 1/30$~Hz, that is, a feature vector $x$ and
its label $y$ is generated every 30~s. 
However, to avoid getting highly correlated samples and to increase diversity of features, we collect a sample only when its label has changed from its previous value.
For example, samples acquired during sleeping are discard automatically as they all
have ``Still'' labels. 
This lowers the actual sampling rate to about $F_s = 1/352$~Hz (or every six minute or so). 
With this low rate, no battery problem was observed.

We use 3-class logistic regression (Table~\ref{tbl:learning algorithms})
with $\lambda=0, b=1, \epsilon^{-1}=0$ and a range of $\eta$ values. 
Repeated experiments with different parameters are time-consuming, and 
we leave the full investigation to the second experiment in a simulated environment. 
\begin{figure}[thb]
\centering
\includegraphics[width=1.0\linewidth]{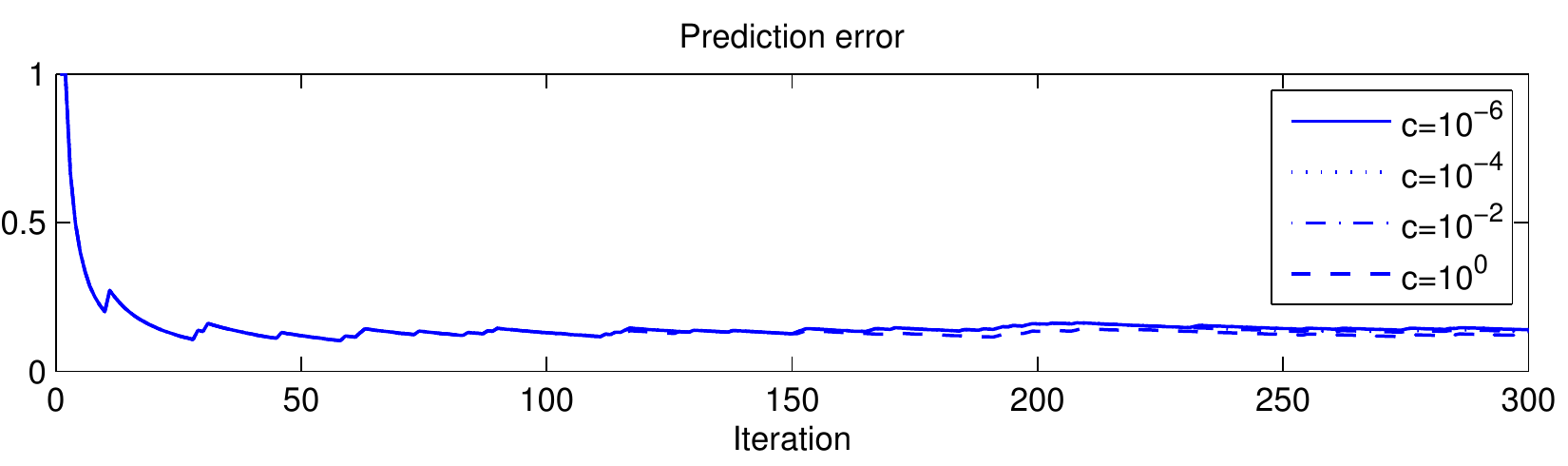}
\caption{Time-averaged error across all devices for activity recognition task. 
}
\label{fig:smartphones}
\end{figure}
In Fig.~\ref{fig:smartphones}, we shows the collective error curves for the first 
300 samples from the 7 devices. 
The error is a time-averaged misclassification error as the learning progresses: 
$
\mathrm{Err}(t) = \frac{1}{t} \sum_{i=1}^{t} I[y^{}_i \neq y^{\mathrm{pred}}_i (w_i)].
$
The error curves for different learning rates (\ref{eq:learning rate})
are very similar, and virtually converge after only 50 samples (=7 samples per device).
This experiment is a proof-of-concept that Crowd-ML can learn a common classifier 
fast, from only a small number of samples per user.

\subsection{Digit/Object Recognition in Simulated Environments}\label{sec:digit recognition}
To evaluate Crowd-ML under various conditions, we perform a series of experiments
on handwritten digit recognition and visual object recognition.
Since the two results are quite similar, we only describe the digit recognition results
(object recognition result is in Appendix~\ref{sec:object recognition}).
The MNIST dataset\footnote{\url{http://yann.lecun.com/exdb/mnist/}} consists of
60000 training and 10000 test images of handwritten digits (0 to 9),
which is a standard benchmark dataset for learning algorithms.
The task is to classify a test image as one of the 10 digit classes.
The images from MNIST data are preprocessed with PCA to have a reduced dimension of 50,
and $L_1$ normalized. 
In this experiment, we compare the performance of centralized, Crowd-ML, and decentralized
learning approaches using the same data and classifier (multiclass logistic regression),
under different conditions such as privacy level $\epsilon$, minibatch size $b$, 
and delays. 
To test the algorithms with a full control of parameters, we run the algorithms
in a simulated environment instead of on a real network. We can therefore 
choose the number of devices and maximum delays arbitrarily.
For simplicity, we set $\tau = \tau_{\mathrm{req}}=\tau_{\mathrm{co}}=\tau_{\mathrm{ci}}$
(Section~\ref{sec:delay}).
The $\tau$ is the maximum delay, and the actually delays
are sampled randomly and uniformly from $[0,\tau]$ for each communication instance.\footnote{We can test with any distribution other than uniform distribution as well.} 

All results in this section are averaged test errors from 10 trials. For each 
trial, assignment of samples, order of devices, perturbation noise, and amounts of delay
are randomized. Test errors are computed as functions of the iteration 
(=the number of samples used), up to five passes through the data.
Hyperparameters $\lambda$ (Table~\ref{tbl:learning algorithms}) and $c$
(\ref{eq:learning rate}) are selected from the averaged test error from 10 trials.
We set the number of devices $M=1000$. Consequently, each
device has $60$ training and $10$ test samples on average.

Fig.~\ref{fig:mnist_M1000_1} compares the performance of the
centralized, crowd, and decentralized learning approaches, without privacy or
delay ($\epsilon^{-1}=0,\; b=1,\; \tau=0$).
The error of centralized batch training is the smallest $(0.1)$, in a tie with 
Crowd-ML.
The error curve of Crowd-ML converges to the same low value as centralized approach.
It shows that incremental update by SGD in Crowd-ML is as accurate as batch learning,
when privacy and delay are not considered.
In contrast, the error curve of decentralized approach converges at a slower rate and
also converges to a high error $(\sim 0.5)$, despite using the same overall number of
samples as other algorithms, due to the lack of data sharing.

\begin{figure}[thb]
\centering
\includegraphics[width=1\linewidth]{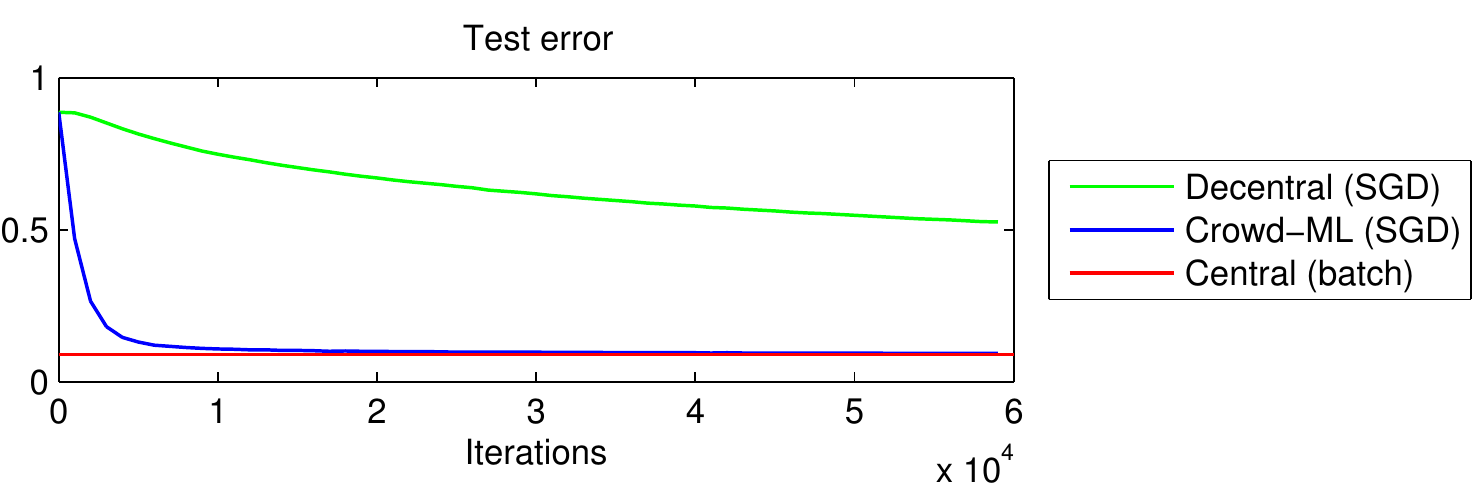}
\caption{Comparison of test error for centralized, crowd, and decentralized
learning approaches, without delay or privacy consideration.
The curves show how error decreases as the number of iteration (=number of samples used)
increases over time. The batch algorithm is not incremental and therefore is a constant. 
}
\label{fig:mnist_M1000_1}
\end{figure}

We perform tests with varying levels of privacy $\epsilon$. The privacy
impacts the centralized approach via (\ref{eq:perturbation of x}) and (\ref{eq:perturbation of y})\footnote{The features and labels for test data are not perturbed.} 
and also Crowd-ML via (\ref{eq:perturbation of g}).
With low privacy ($\epsilon^{-1} \to 0$), the performance of both centralized and
crowd approaches are almost the same as Fig.~\ref{fig:mnist_M1000_1}, 
and we omit the result.
With high privacy ($\epsilon \to 0$), the performance of both approaches degrades
to a unusable level. 
Here we show their performances at $\epsilon^{-1} = 0.1$ in Fig.~\ref{fig:mnist_M1000_2},
where the performance is in a transition state between high and low privacy regions.
Firstly, the centralized and crowd approaches both perform worse than they did
in Fig.~\ref{fig:mnist_M1000_1}, which is the price of privacy preservation. 
Among these results, Crowd-ML with a minibatch size $b=20$ has
the smallest asymptotic error, much below the centralized (batch).
Crowd-ML with $b=1$ and $10$ still achieves similar or better asymptotic error
compared to Central (batch). 
As predicted from Section~\ref{sec:analysis}, increasing the minibatch
size improves the performance of Crowd-ML. 
When SGD is used for centralized approach (Central SGD) 
with perturbed features and labels,
its performance is very poor ($\sim 0.9$) regardless of the minibatch size,
due to the larger noise required to provide the same level $\epsilon$ of privacy
as Crowd-ML.

\begin{figure}[thb]
\centering
\includegraphics[width=1\linewidth]{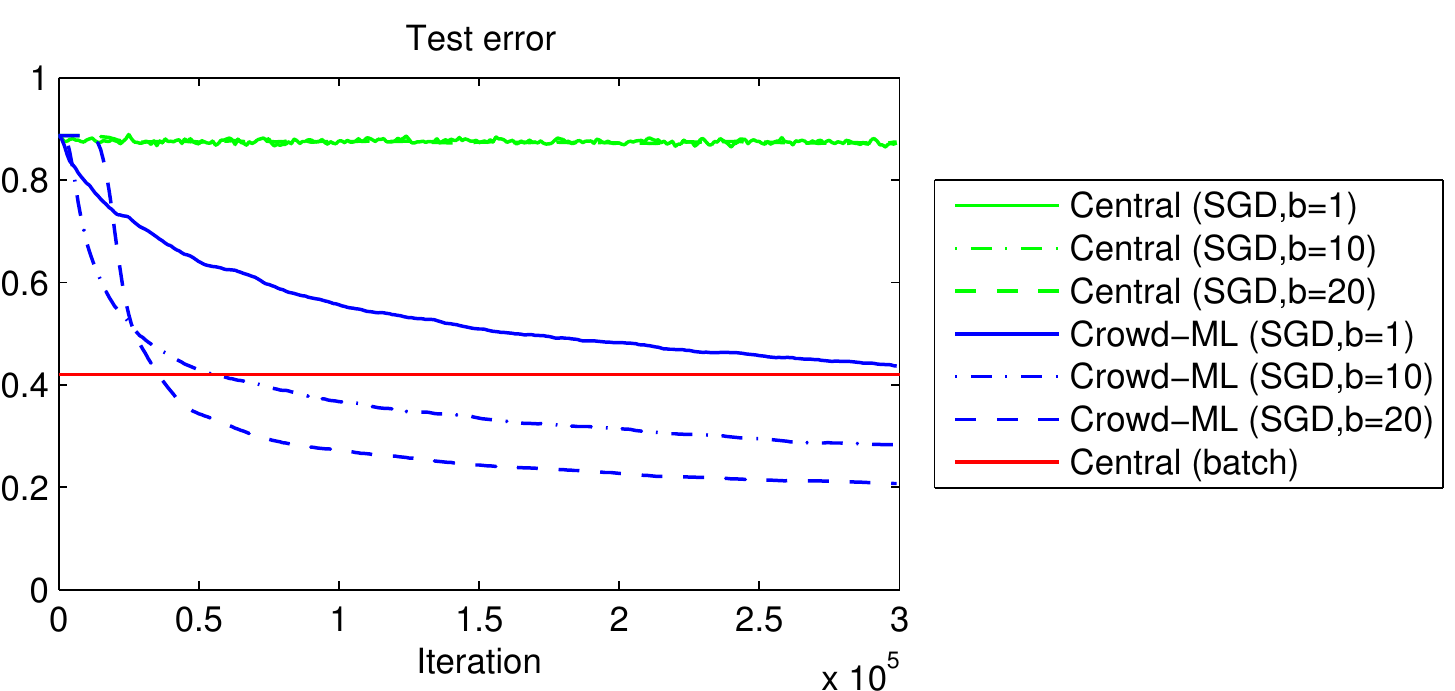}
\caption{Comparison of test error for centralized and crowd learning approaches
with privacy ($\epsilon^{-1} = 0.1$), varying minibatch sizes ($b$), and  
no delay.}
\label{fig:mnist_M1000_2}
\end{figure}

Lastly, we look at the impact of delays on Crowd-ML with privacy $\epsilon^{-1}=0.1$.
We test with different delays in the unit of $\Delta=\tau/(M F_s)$, 
that is, the number of samples generated by all device during the delay of
size $\tau$.
In Fig.~\ref{fig:mnist_M1000_3}, we show the results with two minibatch sizes
($b=1, 20$) and varying delays ($1\Delta,10\Delta,100\Delta,1000\Delta$).
The delay of $1000 \Delta$ means that a maximum of 3$\times$1000 samples are 
generated among the devices, between the time a single device requests a check-out
from the server and the time the server received the check-in from that device,
which is quite large.
Fig.~\ref{fig:mnist_M1000_3} shows that the increase in the delay somewhat 
slows down the convergence with a minibatch size of 1, and the converged value of
error is similar to or worse than Central (batch).
However, it also shows that with a minibatch size of 20, delay has little
effect on the convergence, and the error is much lower than Central (batch).
Note that with the minibatch size of 20, there is a small plateau in the beginning
of error curves, reflecting the fact that the devices are initially waiting
for their minibatches to be filled before computing begins. After this initial
waiting time, the error starts to decrease at a fast rate.


\begin{figure}[thb]
\centering
\includegraphics[width=1\linewidth]{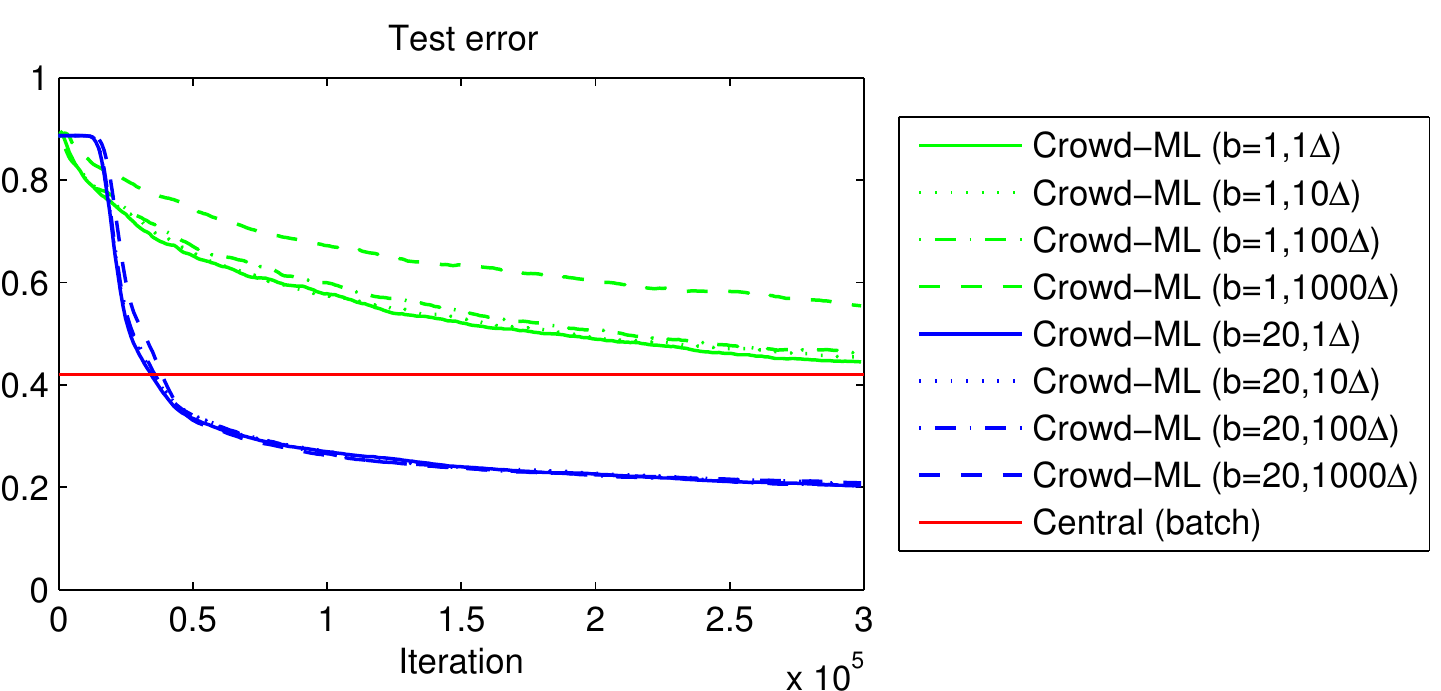}
\caption{Impact of delays on Crowd-ML with privacy ($\epsilon^{-1} = 0.1$),
varying minibatch sizes, and varying delays.
}
\label{fig:mnist_M1000_3}
\end{figure}

\section{Conclusion}\label{sec:conclusion}

In this paper, we proposed Crowd-ML, a machine learning framework for a crowd of
smart devices. 
Compared to previous crowdsensing systems, Crowd-ML is a framework
that integrates sensing, learning, and privacy mechanisms together, 
and can build classifiers or predictors of interest from crowdsensing data 
using computing capability of smart devices.
Algorithmically, Crowd-ML utilizes recent advances in distributed and incremental
learning, and implements strong differentially private mechanisms.
We analyzed Crowd-ML and showed that Crowd-ML can outperform centralized approaches
while providing better privacy and scalability, and can also take advantages of
larger shared data which decentralized approaches cannot.
We implemented a prototype of Crowd-ML and evaluated the framework with a simple
activity recognition task in a real environment as well as larger-scale experiments
in simulated environments which demonstrate the advantages of the design of Crowd-ML.
Crowd-ML is a general framework for a range of different learning algorithms
with crowdsensing data, and is open to further refinements for
specific applications.
\appendix

\subsection{Proof of Theorem 1}
\label{sec:proof1}

In our algorithms, a device receives $w$ from the server and sends averaged gradients
$\hat{g}$ along with other information.
We assume $\|x\|_1 \leq 1$ which can be easily achieved by normalizing the data.
The sensitivity of an averaged gradient for logistic regression is $4/b$ as shown below. 
There are $C$ parameter vectors $w_1,...,w_C$ for multiclass logistic regression.
Let the matrix of gradient vectors corresponding to $C$ parameter vectors be 
\begin{eqnarray*}
g &=& [g_1\;g_2\;\cdots\;\;g_C] = x [P_1\;\cdots\;P_{y}\mathrm{-}1\;\cdots\;P_C] + \lambda [w_1\;\cdots\;w_C]
\\
&=& x M + \lambda [w_1\;\cdots\;w_C],
\end{eqnarray*}
where $P_j = P(y=j|x; w)$ is the posterior probability, and $M$
is a row vector of $P_j$'s. 
Without loss of generality, consider two minibatches $\mathcal{D}$ and $\mathcal{D}'$
that differ in only the first sample $x_1$.
The difference of averaged gradients $\tilde{g}(\mathcal{D})$ and $\tilde{g}'(\mathcal{D}')$ is
\[
\|\tilde{g} - \tilde{g}'\|_1 \leq \frac{1}{b}(\|x_1 M_1\|_1 + \|x'_1 M_1'\|_1)
\leq \frac{4}{b},
\]
To see $\|M_1\|_1 \leq 2$, note that the absolute sum of the entries of $M_1$ 
is $2(1-P_{y_1}) \leq 2$. 
The sensitivity of multiple minibatches $\hat{g}(1), ... , \hat{g}(T)$
is the same as the sensitivity of a single $\hat{g}(t)$, and the
$\epsilon$-differential privacy follows from Proposition 1 of \cite{Dwork:2006:TC}.

\subsection{Proof of Theorem 2}
\label{sec:proof2}

In addition to the averaged gradients, a device sends to the server
the numbers of samples $n_s$, the number of misclassified samples $n_e$, and
the labels counts $n_y^k$.
Perturbation by adding discrete Laplace noise is equivalent to random sampling
by exponential mechanism \cite{McSherry:2007:FOC} with 
$P(\hat{n}_e|n_e) \propto e^{-\frac{\epsilon_e}{2} |\hat{n}_e - n_e|}$, $\hat{n}_e \in \mathbb{Z}$.
If two datasets $\mathcal{D}$ and $\mathcal{D}'$ are different in only one item,
then the score function $d = -|\hat{n}_e - n_e|$ changes at most by 1.
That is, $\max_{\mathcal{D},\mathcal{D}'}\;|d(\hat{n}_e,n_e(\mathcal{D}))-d(\hat{n}_e,n_e(\mathcal{D}'))|=1$. 
As with multiple gradients, the sensitivity of multiples sets of 
($\hat{n}_e$, $\hat{n}_y^k$) is the same as the sensitivity of a single set, 
and $\epsilon_e$-differential privacy follows from Theorem 6 of \cite{McSherry:2007:FOC}.
Proof of $\epsilon_{y^k}$-differential privacy of $n_y^k$ is similar.

\emph{Remark 1:} Unlike the gradient $\hat{g}$, the information ($n_s$, $\hat{n}_e$, $\hat{n}_y^k$)
is not required for learning itself, but for monitoring the progress of each device
on the server side. 
Therefore, $\epsilon_e$ and $\epsilon_{y^k}$ can be set to be very small without
affecting the learning performance, so that $\epsilon = \epsilon_g + \epsilon_e + C\epsilon_{y^k} \approx \epsilon_g$. 

\emph{Remark 2:} $\hat{n}_e$ and $\hat{n}_y^k$ can be negative with a small probability,
but have a limited effect on the estimates of the error rate and the prior at the server.
After receiving $T$ minibatches, the error rate and the prior estimates are
\begin{equation}
\mathrm{Err}^\mathrm{est} = \frac{\sum_i^T \hat{n}_e(i)}{\sum_i^T n_s(i)}
\;\;\mathrm{and}\;\;
P^\mathrm{est}(y=k) = \frac{\sum_i^T \hat{n}_y^k(i)}{\sum_i^T n_s(i)}.
\end{equation}
Since $\hat{n}_e(i)-n_e(i)$ is independent for $i=1,2,...$ and has zero-mean and 
constant variance $\frac{2e^{-\epsilon_e/2}}{(1-e^{-\epsilon_e/2})^2}$ \cite{Inusah:2006:JSPI}, 
the estimate of error rate converge almost surely to the true error rate with vanishing
variances as $T$ increases.
The same can be said of the estimate of prior $P(y)$. 

\subsection{Differential Privacy in Centralized Approach}
\label{sec:proof3}

For completeness of the paper, we also describe the $\epsilon$-differential privacy
mechanisms for the centralized approach.
In the centralized approach, data are directly sent to the server.
Without a privacy mechanism, an adversary can potentially observe all data.
To prevent this, $\epsilon$-differential privacy can be enforced by perturbing
the features 
\begin{equation}\label{eq:perturbation of x}
f(x) = x + z,\;,\;\;P(z) \propto e^{-\frac{\epsilon_x}{2}|z|},
\end{equation}
and also perturbing the labels. 
To perturb labels, we use exponential mechanism to sample a noisy label
$\hat{y}$ given a true label $y$ from
\begin{equation} \label{eq:perturbation of y}
P(\hat{y}|y) \propto e^{\frac{\epsilon_y}{2} d(y,\hat{y})},
\;\;y, \hat{y} \in \{1,...,C\}
\end{equation}
where we use the score function $d(y,\hat{y}) = I[y = \hat{y}]$. 
\begin{theorem}[Feature and label perturbation]
The transmission of $x$ and $y$ by feature perturbation (\ref{eq:perturbation of x})
 and exponential mechanism  (\ref{eq:perturbation of y})
is $\epsilon_x$- and $\epsilon_y$-differentially private.
\end{theorem}
\begin{proof}
Assume $\|x\|_1 \leq 1$. 
Feature transmission is an identity operation and therefore has sensitivity $2$.
For label transmission, the score function $d(\hat{y},y) = I[\hat{y} = y]$ 
changes at most by 1 by changing $y$.
From Proposition 1 of \cite{Dwork:2006:TC} and Theorem 6 of \cite{McSherry:2007:FOC}, respectively, 
we achieve $\epsilon_x$- and $\epsilon_y$-differential privacy of data.
\end{proof}
Note that the sensitivity is independent of the number of features and labels sent,
and we have to add the same level of independent noise to the features 
and apply the same amount of label perturbation. 
An overall $\epsilon$-differential privacy is achieved by
$\epsilon = \epsilon_x + \epsilon_y$.
The required privacy levels $\epsilon_x$ and $\epsilon_y$ can be chosen differently,
and we use $\epsilon_x = \epsilon_y = \epsilon/2$ in the experiments.

\subsection{Experiments with Visual Object Recognition Task}\label{sec:object recognition}

We repeat the experiments in Section~\ref{sec:digit recognition} for
 an object recognition task using
CIFAR-10 dataset, which consists of images of 10 types of objects 
(airplane, automobile, bird, cat, deer, dog, frog, horse, ship, truck)
collected by \cite{Krizhevsky:2009:TR}.
We use 50,000 training and 10,000 test images from CIFAR-10.
To compute features, we use a convolutional neural network
\footnote{\url{https://github.com/jetpacapp/DeepBeliefSDK}} trained using 
ImageNet ILSVRC2010 dataset\footnote{\url{http://www.image-net.org/challenges/LSVRC}},
which consists of 1.2 million images of 1000 categories.
We apply CIFAR-10 images to the network, and use the 4096-dimensional output 
from the last hidden layer of the network as features. 
Those features are preprocessed with PCA to have a reduced dimension of 100,
and are $L_1$ normalized. 
We use the same setting in Section~\ref{sec:digit recognition} to test Crowd-ML 
on this object recognition task. The results are given in Figs.~\ref{fig:cifar10_M1000_1},
\ref{fig:cifar10_M1000_2}, \ref{fig:cifar10_M1000_3}.
The figures are very similar to the handwritten digit recognition task 
(Figs.~\ref{fig:mnist_M1000_1}, \ref{fig:mnist_M1000_2}, \ref{fig:mnist_M1000_3}),
except that the error is larger (e.g., $0.3$ in Fig.~\ref{fig:cifar10_M1000_1})
than the error for digit recognition ($0.1$ in Fig.~\ref{fig:mnist_M1000_1}).
This is because CIFAR dataset is more challenging than MNIST due to variations
in color, pose, view point, and background of object images. 

\begin{figure}[thb]
\centering
\includegraphics[width=0.95\linewidth]{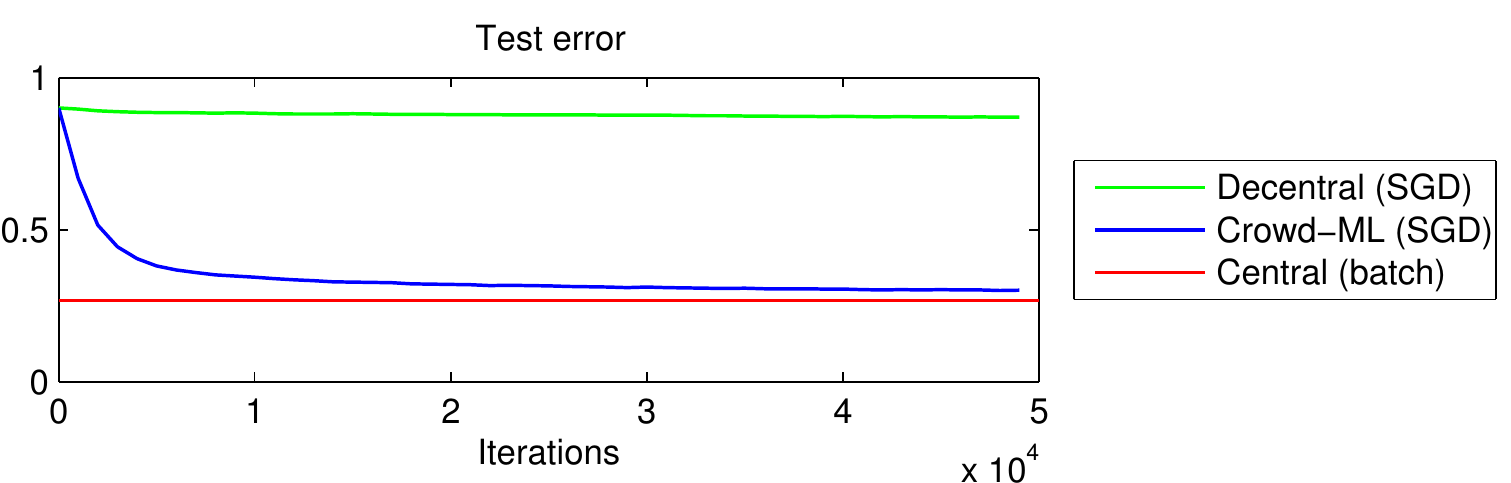}
\caption{Comparison of test error for centralized, crowd, and decentralized
learning approaches, without delay or privacy consideration.
}
\label{fig:cifar10_M1000_1}
\end{figure}

\begin{figure}[thb]
\centering
\includegraphics[width=0.95\linewidth]{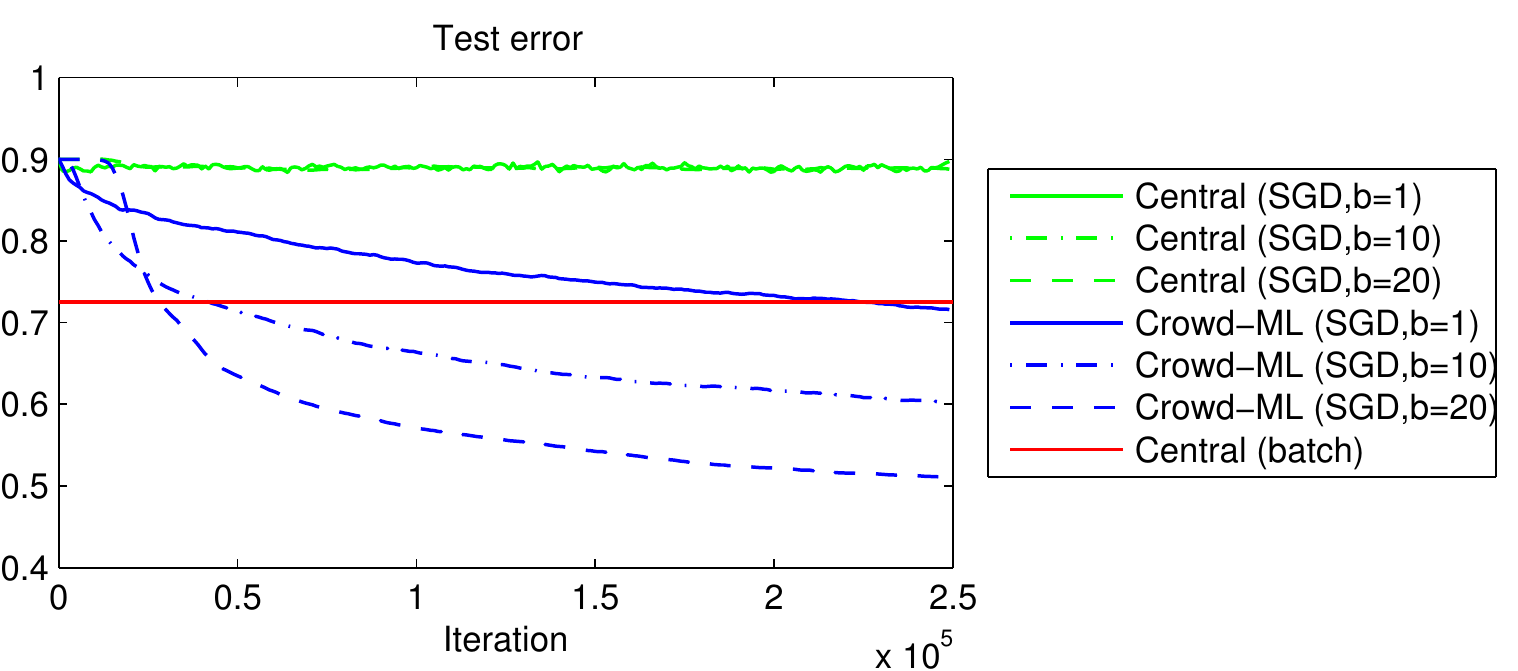}
\caption{Comparison of test error for centralized and crowd learning approaches
with privacy ($\epsilon^{-1} = 0.1$), varying minibatch sizes ($b$), and  
no delay.}
\label{fig:cifar10_M1000_2}
\end{figure}

\begin{figure}[thb]
\centering
\includegraphics[width=0.95\linewidth]{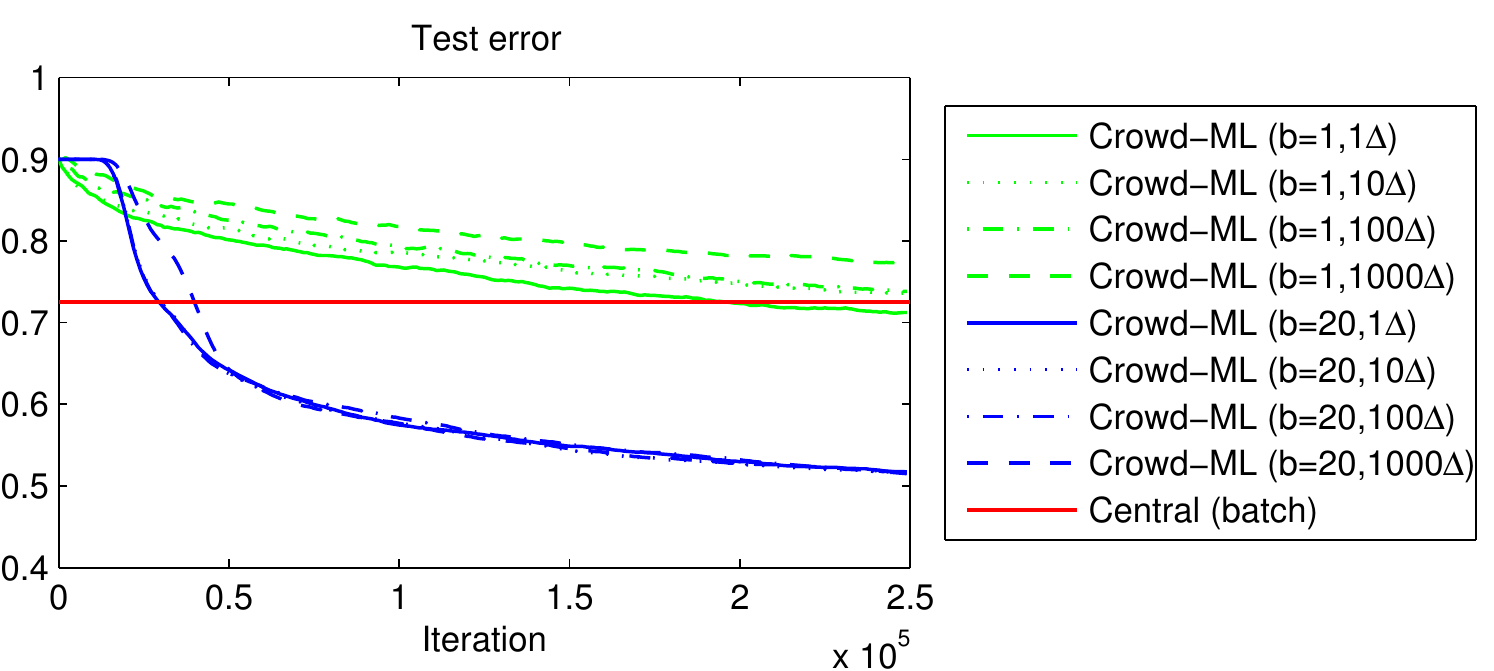}
\caption{Impact of delays on Crowd-ML with privacy ($\epsilon^{-1} = 0.1$),
varying minibatch sizes, and varying delays.
}
\label{fig:cifar10_M1000_3}
\end{figure}

\bibliographystyle{IEEEtran}
\bibliography{icdcs15_jh,refs_ac,nips14_1_jh,mobicase12_jh}


\end{document}